\documentclass[11pt,onecolumn]{article}

\usepackage{cvpr}
\usepackage{times}
\usepackage{epsfig}
\usepackage{graphicx}
\usepackage{amsmath}
\usepackage{amssymb}

\usepackage[ruled]{algorithm2e}
\usepackage{dsfont}
\usepackage{fixltx2e}

\usepackage{color}
\usepackage{multirow}
\usepackage{rotating,url}
\usepackage{balance}
\usepackage{fancyhdr}

\usepackage[pagebackref=true,breaklinks=true,letterpaper=true,colorlinks,bookmarks=false]{hyperref}
\newtheorem{theorem}{Theorem}[section]

\newtheorem{lemma}[theorem]{Lemma}
\newtheorem{proposition}[theorem]{Proposition}
\newtheorem{corollary}[theorem]{Corollary}

\newenvironment{proof}[1][Proof]{\begin{trivlist}
\item[\hskip \labelsep {\bfseries #1}]}{\end{trivlist}}

\newcommand{\qed}{\nobreak \ifvmode \relax \else
      \ifdim\lastskip<1.5em \hskip-\lastskip
      \hskip1.5em plus0em minus0.5em \fi \nobreak
      \vrule height0.75em width0.5em depth0.25em\fi}

\cvprfinalcopy 

\def\cvprPaperID{686} 

\begin{document}

\title{Superpixelizing Binary MRF for Image Labeling Problems}

\author{Junyan Wang and
Sai-Kit Yeung\\
Singapore University of Technology and Design,\\
8 Somapah Road,\\
Singapore, 487372\\
{\tt\small \{junyan\_wang,saikit\}@sutd.edu.sg}
}

\maketitle

\begin{abstract}
{Superpixels have become prevalent in computer vision. They have been used to achieve satisfactory performance at a significantly smaller computational cost for various tasks. People have also combined superpixels with Markov random field (MRF) models. However, it often takes additional effort to formulate MRF on superpixel-level, and to the best of our knowledge there exists no principled approach to obtain this formulation. In this paper, we show how generic pixel-level binary MRF model can be solved in the superpixel space. As the main contribution of this paper, we show that a superpixel-level MRF can be derived from the pixel-level MRF by substituting the superpixel representation of the pixelwise label into the original pixel-level MRF energy.} The resultant superpixel-level MRF energy also remains submodular for a submodular pixel-level MRF. The derived formula hence gives us a handy way to formulate MRF energy in superpixel-level. In the experiments, we demonstrate the efficacy of our approach on several computer vision problems.
\end{abstract}

\section{Introduction}

	
	%

Many computer vision problems can be cast as image labeling problems. Markov random field (MRF) is a general-purpose optimization model for image labeling~\cite{Gemans84MRF,StanLi09MRFbook}. Recent progress on MRF shows its prominent advantages for solving various computer vision and machine learning problems~\cite{Boykov01GraphCut,Boykov2004experimentalGC,Kolmogorov2007non_submodular,Szeliski2008MRFcompare,kappes2013comparative}.

Superpixelization, a.k.a. over-segmentation, is an intuitive yet effective approach to reducing the dimensionality of the image space for computer vision problems~\cite{Vincent1991watersheds,Comaniciu2002meanshift,Shi2000Ncut,Vedaldi2008QuickShift,Levinshtein2009turbopixels,Veksler2010superpixels,Wang2012Vcells,Achanta2012Slic}, and it has been used in combination with MRF~\cite{Zitnick07SPStereo,Fulkerson09class_SPCRF,Vazquez2010multiple,Nowozin2010parameter,Tighe10SuperParsing,Ren2012RGBDlabeling,Khan14SPMRF_RGBD,Rantalankila2014GenProposal}. Superpixels can be used to speed up the image labeling and they often form natural regularization to the labeling problems. However, it often takes significant effort to reformulate the original pixel-level MRF problem into a superpixel-level MRF problem. To the best of our knowledge, there exists no principled approach to obtain the superpixel-level MRF. 


In this paper, we show how to minimize a given generic pixel-level binary MRF energy in the superpixel space. 
To this effect, we first represent pixelwise label by superpixel label. We then substitute this superpixel representation into the pixel-level MRF energy. As the main contribution of this paper, we show that superpixel-level MRF energy can be derived from the pixel-level MRF. In addition, the derived superpixel-level MRF is submodular if the original MRF model is submodular. Fig.~\ref{Final_SPMRF2} illustrates the main idea of this paper. We demonstrate the usefulness of our technique on three representative image labeling problems. 

\begin{figure}
\centering
\vspace{0.6cm}
\includegraphics[width=0.7\columnwidth]{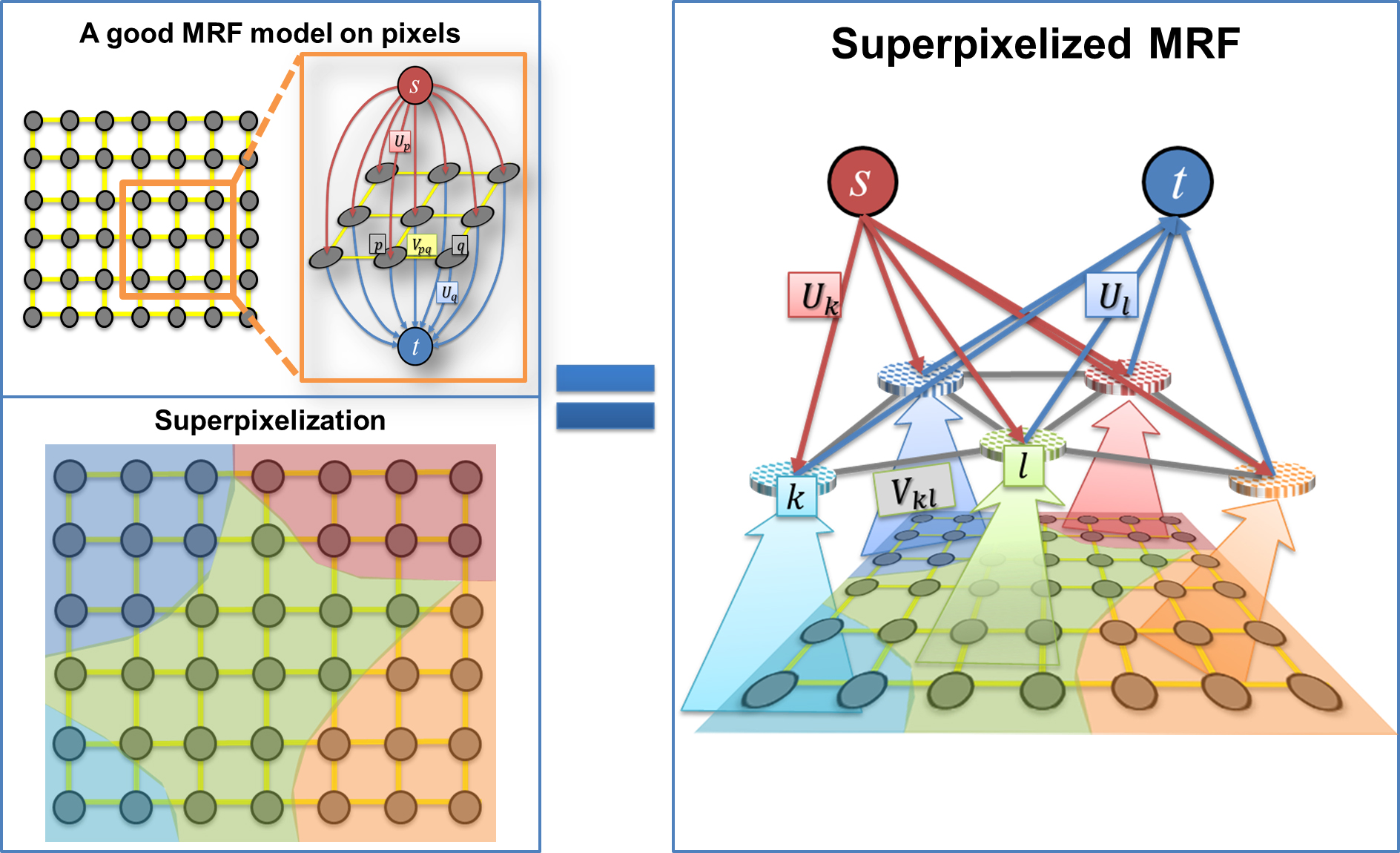}
\caption{Superpixelizing MRF and preserving submodularity (representable via s-t graph). $k$ and $l$ are superpixel indices. $U_k, U_l$ and $V_{kl}$ are the unary and pairwise potentials.}\vspace{-0.5cm}
\end{figure}\label{Final_SPMRF2}
The remaining of this paper is organized as follows. In the next section, we will review the generic form of the second order binary MRF. In section \ref{SEC:D-SPMRF}, we will present the technique that we used to superpixelize the MRF energy. In section \ref{SEC:APP}, we briefly introduce the three applications we considered in this work. In section \ref{SEC:Exp}, we present the experimental results of the respective applications with comparison to the state-of-the-art methods. In section \ref{SEC:Concl}, we conclude the paper and suggest some future works. 


\section{Binary MRF model for image labeling}\label{SEC:BMRF}

The generic second order MRF model can be written as follows:
\begin{equation}\label{eqn:MRF}
\min_f \sum_{p\in \mathcal{P}} U_p(f_p)+\sum_{(p,q)\in \mathcal{N}} V_{pq}(f_p,f_q), 
\end{equation}
where $f_p$ and $f_q$ are the pixel-wise labels over the image, we consider the label values to be either 1 or 0 henceforth. $\mathcal{P}$ is the set of all pixels in the image, and $\mathcal{N}$ is a neighborhood system. $U_p(\cdot)$ is known as the unary term or data-fidelity term. $V_{pq}(\cdot,\cdot)$ is the pairwise potential that is often used to model the pairwise relationship between the labels on neighboring pixels.

For binary-label problem, the unary term can be written more explicitly as
\begin{equation}
U_p(f_p)=\left\{\begin{array}{cc}
          w^1_p, & \hbox{ if } f_p = 1 \\
          w^0_p, & \hbox{ if } f_p = 0 \\
        \end{array}\right.,
\end{equation}
or
\begin{equation}
U_p(f_p) = w^1_p f_p + w^0_p (1-f_p) = (w^1_p - w^0_p) f_p + w^0_p.
\end{equation}

The generic form of the pairwise term can be written as
\begin{equation}
V_{pq}(f_p,f_q)=\left\{\begin{array}{cl}
          w^{00}_{pq}, & \hbox{ if } f_p=f_q=0 \\
          w^{01}_{pq}, & \hbox{ if } f_p=0,~f_q=1\\
          w^{10}_{pq}, & \hbox{ if } f_p=1,~f_q=0\\
          w^{11}_{pq}, & \hbox{ if } f_p=f_q=1\\
        \end{array}\right..
\end{equation}
Thus, $V_{pq}=w^{00}_{pq}\overline{f_p}\overline{f_q}+w^{01}_{pq}\overline{f_p}{f_q}+w_{pq}^{10}{f_p}\overline{f_q}+w_{pq}^{11}{f_p}{f_q}$. 

To sum up, we may rewrite the generic binary label MRF model explicitly as follows:
\begin{equation}\label{EQ:MRF}
\begin{split}
&\min_f \sum_{p\in\mathcal{P}} w_p f_p+\sum_{(p,q)\in \mathcal{N}}\left( w^{00}_{pq}\overline{f_p}\overline{f_q}+w^{01}_{pq}\overline{f_p}{f_q}\right.\\
&\hspace{3.4cm}\left.+w_{pq}^{10}{f_p}\overline{f_q}+w_{pq}^{11}{f_p}{f_q}\right).
\end{split} 
\end{equation}
where $w_p = w^1_p - w^0_p$. Note that we have omitted the constant terms. 

It has been proven in \cite{Kolmogorov04GraphCut} that if $w^{00}_{pq}+w^{11}_{pq}\leq w^{01}_{pq}+w^{10}_{pq}$, the binary labeling problem is submodular and hence can be solved by graph cuts exactly. We will focus on submodular MRF model in this paper. 

\section{Superpixelizing MRF}\label{SEC:D-SPMRF}




\subsection{Superpixel representation of pixel labeling}

\begin{figure}
\centering
\vspace{0.5cm}
\includegraphics[width=0.7\columnwidth]{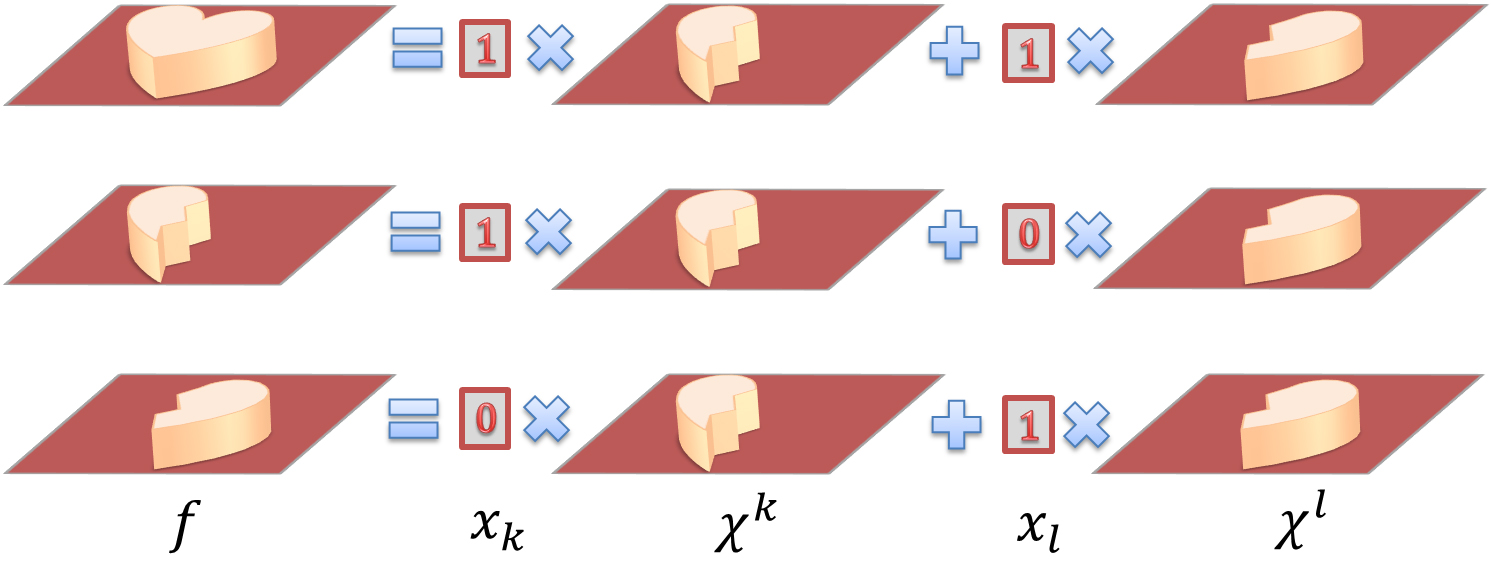}
\caption{Superpixel (cake-cutting) representation of 0-1 labeling}\label{FIG:heart-cake}
\end{figure}

Superpixels are essentially adjacent and non-overlapping image regions. We can denote each superpixel $k$ by one indicator function $\chi^k$ defined on the entire image domain, and the superpixel indicator function $\chi^k$ would satisfy:
\begin{equation}
\chi^k_p=1,\hbox{~if~} p\in\Omega_k,~\hbox{and}~
\chi^k_p\chi^l_p = 0,\hbox{~if~} k\neq l,
\end{equation}
where we concatenated  $\chi^k$ to be $\chi^k_p=\chi^k(\mathbf{z}_p)$, $\mathbf{z}_p$ is the pixel location, and $\Omega_k$ is actually the set of all pixels belonging to the $k$-th superpixel.

Based on the above representation of superpixels, the pixelwise labeling over the image can be represented using the superpixel labels as
\begin{equation}\label{EQ:SPlabel}
f_p = \sum_{k=1}^K x_k\chi^k_p,
\end{equation}
where we considered the concatenated form of pixel labeling $f_{p}=f(\mathbf{z}_p)$, $x_k$ is the superpixel label. This superpixel representation of image labeling is also illustrated in Fig.~\ref{FIG:heart-cake}, where we only consider two superpixels, and the label value is either 0 or 1. 

We now derive some basic properties from this superpixel representation of image labeling. These properties will be useful in the derivation of the superpixel-level MRF energy.
\begin{lemma}\label{LM:uniq_SP}
For any $p\in{\Omega}_l$,  $f_p = x_{l}\chi^{l}_p=x_l$， where $\Omega_l = \{p|\chi^l_{p}=1\}$.
\end{lemma}

The above lemma implies the following property.
\begin{corollary}\label{Col:invL}
$\overline{f_p} = \sum_{k=1}^{K} \overline{x_k}\chi^k_p$
\end{corollary}

We defer their proofs to Appendix. 

\subsection{The derivation of superpixel-level MRF} 

With the superpixel representation of image labeling, we are able to write down a naive form of superpixelized MRF energy minimization problem:
\begin{equation}
\begin{split}
\min_{f,\mathbf{x}} &~E_1(f)+E_2(f)\\
\hbox{s.t.: }&f_p = \sum_{k=1}^K x_k\chi_p^k,~p\in\mathcal{P},
\end{split}
\end{equation}
where $E_1$ and $E_2$ are the total unary and pairwise potential terms in the original MRF energy, and $\mathbf{x} = \{x_k|k=1,2,...,K\}$ is the set of all superpixel labels.

This problem is equivalent to
\begin{equation}\label{EQ:PMRF_SP}
\min_{\mathbf{x}} E_1\left(\sum_{k=1}^K x_k\chi_p^k\right)+E_2\left(\sum_{k=1}^K x_k\chi_p^k\right).
\end{equation}
The above discrete optimization problems may appear to be difficult to solve.  

The main contribution of this paper can be written as a proposition as follows.
\begin{proposition}\label{PROP:main}
Given that $f_p = \sum_{k=1}^K x_k\chi^k_p$, the energy in Eq. (\ref{EQ:PMRF_SP}) can be written as an MRF energy defined on superpixels, namely
\begin{equation}\label{EQ:SPMRF}
\begin{split}
&\sum_{k=1}^K  \hat{\omega}_kx_k+\hspace{-0.2cm}\sum_{\{k=1,l=1|k\neq l\}}^{K,K}\big(\omega^{00}_{kl}\overline{x_k}\overline{x_l}+\omega_{kl}^{01}{x_k}\overline{x_l}\\
&\hspace{3.3cm}+\omega_{kl}^{10}{x_k}\overline{x_l}+\omega_{kl}^{11}{x_k}{x_l}\big),\\
\end{split}
\end{equation}
where the first term is the unary term, i.e. $U_k$, and the second term is the pairwise potential, i.e. $V_{kl}$. $\hat{\omega}_k = \omega_k - \omega^{00}_{k}+ \omega^{11}_{k}$,  $\omega_k = \sum_p w_p\chi^k_p$, $\omega_{k}^{mn} = \sum\limits_{(p,q)\in \mathcal{N}} w^{mn}_{pq} \chi_p^k \chi_q^k$, and $\omega_{kl}^{mn} = \sum\limits_{(p,q)\in \mathcal{N}} w^{mn}_{pq} \chi_p^k \chi_q^l$, $(m,n)\in \{0,1\}$. 
\end{proposition}

The proof of this proposition is deferred to the Appendix. Eq.~\ref{EQ:SPMRF} gives us a formula which relates the MRF energy between superpixel and pixel explicitly. With this formula we can build the MRF for superpixels using the MRF in pixel level regardless of the underlying applications. To understand the resultant pairwise potential more in-depth, we elaborate on the relationship between the pairwise potentials before and after superpixelization. 

According to Eq. (\ref{EQ:MRF}), the pairwise potential for the pixel-level MRF can be written as:
\begin{equation}\label{EQ:PPbfSP}
V_{pq} = w^{00}_{pq}\overline{f_p}\overline{f_q}+w^{01}_{pq}\overline{f_p}{f_q}+w^{10}_{pq}{f_p}\overline{f_q}+w^{11}_{pq}{f_p}{f_q}
\end{equation}

Likewise, the pairwise potential for the superpixel-level MRF can be written as: 
\begin{equation}\label{EQ:PPafSP}
V_{kl} = \omega^{00}_{kl}\overline{x_k}\overline{x_l}+\omega^{01}_{kl}\overline{x_k}{x_l}+\omega^{10}_{kl}{x_k}\overline{x_l}+w^{11}_{kl}{x_k}{x_l}.
\end{equation}
\vspace{-0.5cm}
\begin{corollary}\label{Col:V}
Given the pairwise potentials defined in Eq. (\ref{EQ:PPbfSP}) and Eq. (\ref{EQ:PPafSP}), we have the following relationship between them:
\begin{equation}
V_{kl} = \sum_{\{p,q\}\in\mathcal{N}} V_{pq},~ \hbox{for~} p\in\Omega_k~\hbox{and}~ q\in\Omega_l,~k\neq l
\end{equation}
where $\Omega_k$ and $\Omega_l$ are different superpixels.
\end{corollary}

We illustrate the construction of the pairwise potential in Fig. \ref{Fig:V}. In addition, to ensure the solvability of the resultant problem, it is important to maintain the submodularity of the superpixel-level MRF model. We find that the derived superpixel-level MRF is indeed submodular if the original pixel-level MRF is submodular.
\begin{figure}[!t]
\centering
\vspace{0.5cm}
\includegraphics[width=0.7\columnwidth]{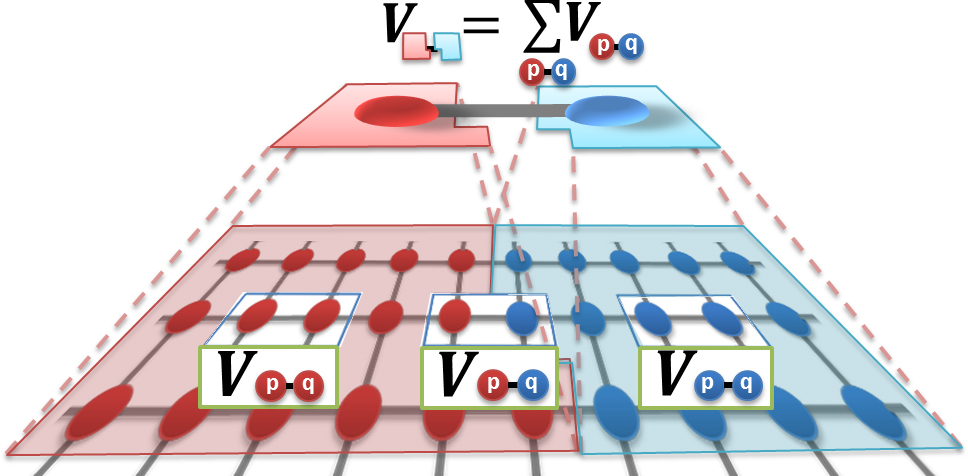}\\\vspace{5pt}
\caption{Visualization of the relationship between the pairwise potential of superpixel-level and pixel level in corollary \ref{Col:V}. Note that the other two types of pixel-level pairwise potentials will contribute to the superpixel-level unary term as shown in proposition \ref{PROP:main}.}\label{Fig:V}
\end{figure}

\begin{proposition}\label{Prop:Regularity}
If the pairwise potential satisfies the regularity inequality, namely
\begin{equation}\label{EQ:regineq}
w^{00}_{pq}+w^{11}_{pq}\leq w^{01}_{pq}+w^{10}_{pq}, 
\end{equation}
then the following inequality holds as well.
\begin{equation}
\omega^{00}_{kl}+\omega^{11}_{kl}\leq \omega^{01}_{kl}+\omega^{10}_{kl}, 
\end{equation}
\end{proposition}
The proof of this proposition is deferred to the Appendix.

Comparing with the original MRF model in Eq. (\ref{EQ:MRF}), the superpixel MRF in Eq. (\ref{EQ:SPMRF}) requires significantly smaller graph for the same problem. 

\subsection{Superpixelizing the Potts model}
One common form of binary MRF energy is the Potts model as follows:
\begin{equation}\label{EQ:Potts}
\min_f \sum_p w_p f_p+\sum_{(p,q)\in \mathcal{N}} w_{pq}|f_p-f_q|^2. 
\end{equation} 

We are particularly interested in the superpixel energy form of the above energy. First, we can rewrite the energy in the general form as in Eq. (\ref{EQ:MRF}). Let $E_2^{Potts}=\sum_{(p,q)\in N} w_{pq}|f_p-f_q|^2$, we have:
\begin{equation}\label{EQ:E2potts}
\begin{split}
E_2^{Potts} &= \sum_{(p,q)\in\mathcal{N}} w_{pq}|f_p-f_q|^2\\
& = \sum_{(p,q)\in \mathcal{N}}\Big( w_{pq}f_p\overline{f_q}+w_{pq}\overline{f_p}{f_q}\Big).
\end{split}
\end{equation}

Thus, the corresponding superpixel MRF is the following:
\begin{equation}\label{EQ:SP_potts}
\begin{split}
&\min_{\mathbf{x}} \sum_{k=1}^K \omega_kx_k+\sum_{k=1,l=1}^{K,K}\Big(\omega_{kl}{x_k}\overline{x_l}+\omega_{kl}\overline{x_k}{x_l}\Big)\\
\Leftrightarrow&\min_{\mathbf{x}} \sum_{k=1}^K \omega_kx_k+\sum_{k\neq l}^{K,K}\Big(\omega_{kl}{x_k}\overline{x_l}+\omega_{kl}\overline{x_k}{x_l}\Big)\\
\Leftrightarrow&\min_{\mathbf{x}} \sum_{k=1}^K \omega_kx_k+\sum_{k\neq l}^{K,K}\omega_{kl}|{x_k}-{x_l}|^2,
\end{split}
\end{equation}
where $w^{00}_{pq}=w^{11}_{pq}=0$, $\omega_k=\sum_{p}w_p\chi_p^k$ and $w_{kl}=\sum_{pq}w_{pq}\chi_p^k\chi_q^l$.

\subsection{Superpixel MRF for segmentation with detected edges}
It has been shown that the segmentation with an MRF model can be made very effective for object segmentation if the detected edge is incorporated in the model~\cite{ajayTPAMI12SegFixation}. The main contribution in their model is using edge map to form the pairwise potential in the Potts model as follows:
\begin{equation}\label{EQ:edgeMRF}
V_{pq}(f_p,f_q) =  w^{e}_{pq} \big|f_p-f_q\big|^2,
\end{equation}
where $f_p$ and $f_q$ are the label variables, they are either 0 or 1, and $w^{e}_{pq}$ is defined as:
\begin{equation}\label{EQ:weq}
w^{e}_{pq}=\left\{\begin{array}{lr}
                    \exp(-5I_e(p,q)),~ & I_e(p,q)\neq0 \\
                    20,~ & \hbox{Otherwise}
                  \end{array},
\right.
\end{equation}
in which $I_e(p,q)$ is 1 if either $p$ or $q$ is on edge~\cite{Arbelaez2011ContourDetectSeg,Dollar2013structured}.

As their method targets at automatic object segmentation, computational efficiency is a critical concern. We propose to superpixelize their MRF energy to gain similar performance of segmentation at a much smaller computational cost. Note that it is also not straightforward to reasonably incorporate edge detection in an MRF defined on superpixels. 


Interestingly, Ren et al. \cite{Ren2012RGBDlabeling} proposed a superpixel MRF with detected edge. Nevertheless, the explicit relationship between the superpixel-level and pixel-level pairwise potential was not given. Thus, the optimal formulation for this term may be obscure. With our superpixelization formula for Potts model established in Eq. (\ref{EQ:SP_potts}), the explicit form of the pairwise potential for the edge based superpixel MRF can be easily derived from Eq. (\ref{EQ:edgeMRF}).   

\section{Applications}\label{SEC:APP}
In this section, we briefly review the applications that we considered in the experimental evaluation.

\subsection{Interactive image segmentation}
Interactive image segmentation is a typical application of MRF model~\cite{BoykovJolly01GMM-MRF}. It has been successfully incorporated in the system of image cutout~\cite{lazysnapping,Rother04GrabCut}. The image cutout is now composed of three components, object masking, boundary editing and alpha-matting~\cite{lazysnapping}. In this paper, we consider the basic module of an interactive segmentation system, i.e. the box and seeds controlled object masking. Recent developments on MRF model based interactive image segmentation are mainly focused on the unary term~\cite{chen2012adaptive,TangMeng13OneCut,Wu14MILCut}. Since the superpixelization of the unary term is relatively straightforward, in this work we consider the effectiveness of our MRF superpixelization for the state-of-the-art pairwise potential~\cite{ajayTPAMI12SegFixation}.  

\subsection{Segmentation propagation in video cutout}
Interactive video cutout is a useful tool in video editing and compositing~\cite{Wang05IVC_SIGGRAPH,Bai09VideoSnapCut_SIGGRAPH,Zhong2012UDC_SIGGRAPHAsia}. It usually begins with an interactive key frame segmentation, followed by segmentation propagation. The segmentation propagation step automatically generates segmentations of the subsequent frames by motion estimation, foreground-background classification and MRF based optimization. On the one hand, since the video cutout is usually a tedious work, the efficiency of the segmentation propagation step is crucial to the usability of such system. On the other hand, accuracy is of utmost importance in video cutout. In other words, the computational cost should not be reduced at any cost of accuracy. We propose to superpixelize the original MRF model in segmentation propagation to safely reduce the computational cost. 

\subsection{Automatic segmentation proposal generation}
Automatic segmentation proposal generation is a relatively new topic in computer vision~\cite{Carreira2010parametricGC,Endres2010ObjProp,Van2011segmentation,Rantalankila2014GenProposal}. It aims to integrate the object detection with object segmentation. The main idea is to generate a pool of segmentation results, as a substitute to sliding windows, to feed into the object detector. The major challenge is that this method can result in very high computational cost. Normally, thousands of proposals will be generated for each image to ensure a satisfactory recall~\cite{Rantalankila2014GenProposal}. Although superpixels have been adopted to reduce the computational burden in the existing frameworks, the relationship between the formulated MRF models for segmentation with superpixels and advanced celebrated pixel-level MRF models \cite{ajayTPAMI12SegFixation,Veksler08StarShapePrior,Kumar06DRF} remains mysterious. The state-of-the-art segmentation methods are generally working on pixel level~\cite{ajayTPAMI12SegFixation,TangMeng13OneCut}. Thus, we propose to superpixelize the existing state-of-the-art pixel-level MRF, such as in~\cite{ajayTPAMI12SegFixation}, for generating object proposals. Witnessing the effectiveness of the pixel-level MRF models, we can expect the similarly successful object proposal generation with the superpixelized MRF.

\section{Experiments}\label{SEC:Exp}

In this section, we evaluate our method for the aforementioned applications. The methods are implemented using MATLAB. We will release our code and datasets upon acceptance. For preprocessing, we adopt a classic edge detection method~\cite{Martin2004LearnEdgeDetect} and a popular superpixelization method~\cite{Levinshtein2009turbopixels}. 

\begin{figure*}
\begin{center}
\vspace{0.1cm}
\begin{tabular}
{
@{\hspace{0mm}}c@{\hspace{1mm}}c@{\hspace{1mm}}c @{\hspace{1mm}}c
@{\hspace{1mm}}c@{\hspace{1mm}}c@{\hspace{1mm}}c @{\hspace{1mm}}c
}
\includegraphics[width=0.22\textwidth]{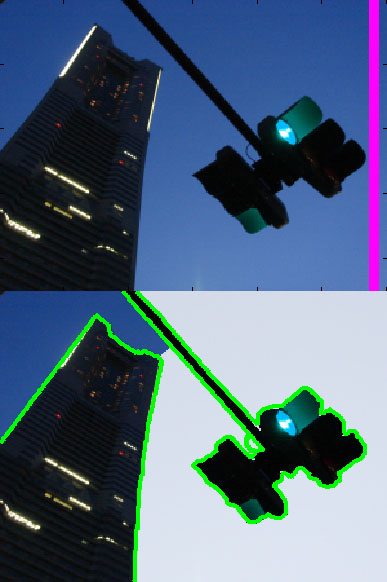} &
 \includegraphics[width=0.22\textwidth]{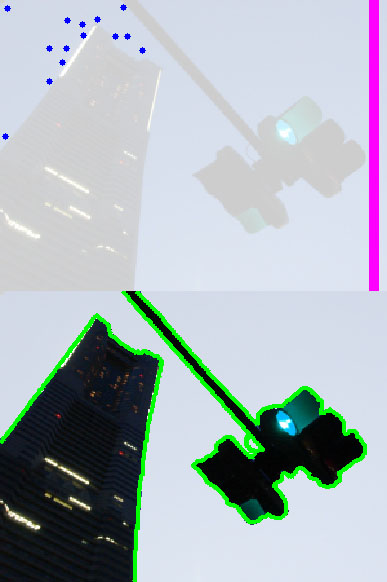}&
 \includegraphics[width=0.22\textwidth]{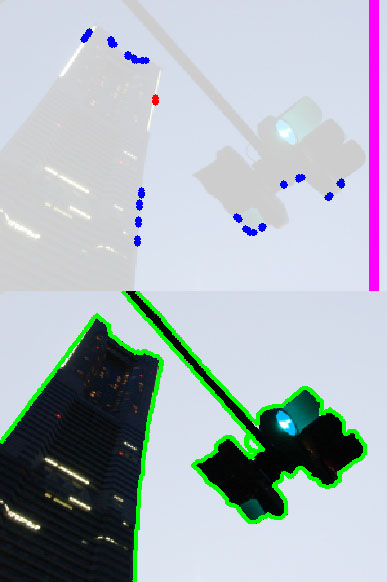}&
 \includegraphics[width=0.22\textwidth]{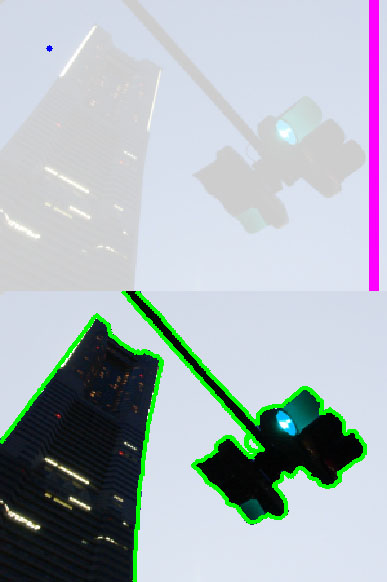}\\
\includegraphics[width=0.22\textwidth]{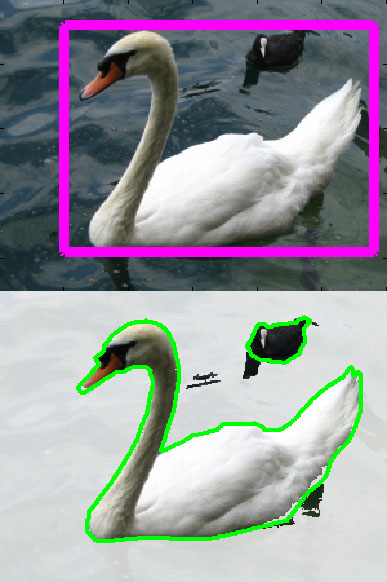} &
 \includegraphics[width=0.22\textwidth]{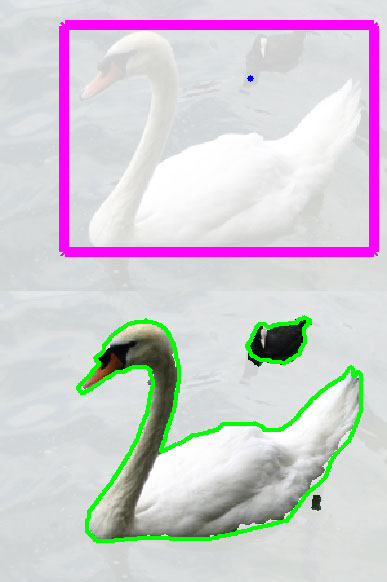}&
 \includegraphics[width=0.22\textwidth]{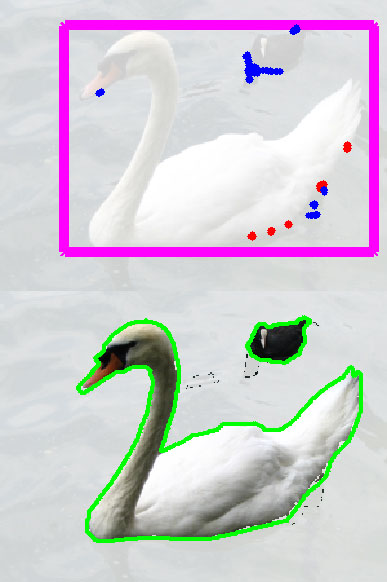}&
 \includegraphics[width=0.22\textwidth]{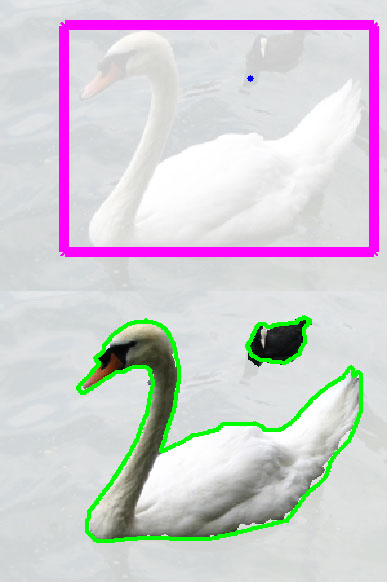}\\
{\scriptsize AdaFBC~\cite{chen2012adaptive}} & {\scriptsize AdaFBC + S-SPGC} & {\scriptsize AdaFBC + Aseg on pixels~\cite{ajayTPAMI12SegFixation}} & {\scriptsize Our method} \\
  \end{tabular}
\end{center}\vspace{-0.2cm}
  \caption{Two sets of results for interactive image segmentation. The top rows show the initial box and the seeds provided by robot user. The images containing seeds have been whitened. The images are better seen by zooming in. Notice that our method produces similar or better results with significant less user effort. 
  }\label{FIG:ExpReg}
\end{figure*}

\subsection{Datasets and evaluation metrics}

\paragraph{Interactive image segmentation.} For evaluating our methods with bounding box input, we adopt the dataset used in~\cite{chen2012adaptive}. It is a subset of the Weizmann segmentation dataset, and it contains 100 images with relatively strong object-background contrast. We use the bounding box provided with the dataset followed by seeds input generated with the robotuser~\cite{kohli2012user} as the user input.

We measure the performance of the methods with segmentation accuracy and the corresponding user effort to achieve the accuracy. The segmentation accuracy is defined as the \emph{overlapping ratio} between the result and the ground truth, i.e. $size(H_o\cap H^*)/size(H_o\cup H^*)$ where $H_o$ is the segmentation result and $H^*$ is the ground-truth segmentation. The \emph{user effort} is measured by the total geodesic distance of the seed points, i.e. the sum of the minimal pairwise distance over the point set. 
In this experiment, the number of superpixels is around 800 for all the images.  This set of experiments were conducted on a PC with Intel Core i5-450M (2.4GHz) processor and 8GB memory.
 
\paragraph{Segmentation propagation in video cutout.} There is one benchmark dataset for interactive video cutout~\cite{Zhong2012UDC_SIGGRAPHAsia}. In the experiment, we evaluate our method on their testing sequences which consists of 6 video sequences with 2070 frames. Since the video cutout task tolerates very little error, we measure the performance of the methods using the \emph{boundary deviation}, i.e. the average distance from the object boundary in the segmentation result and the ground truth object boundary. We also use more superpixels, around 3200, in this experiment. This set of experiments were conducted on a PC with Intel Core i7-4700MQ (2.4GHz) processor and 32GB memory.

\paragraph{Segmentation proposal generation.} 
In this experiment, we use the code shared with \cite{Rantalankila2014GenProposal}. We evaluate on the same test dataset they experimented on, which is part of the PASCAL VOC 2012 segmentation challenge. In the comparison we did not include superpixel refinement even it was proven useful for the task. 
In brief, we directly use the SLIC [cite] in the comparison and replace the pairwise potential of the superpixel MRF constructed in \cite{Rantalankila2014GenProposal} with the pairwise potential superpixelized from Eq. (\ref{EQ:edgeMRF}). We adopt the maximum overlapping ratio of the generated proposal for each object in each image as the evaluation metric. This set of experiments were conducted on a PC with Intel Core i7-4700MQ (2.4GHz) processor and 32GB memory. 

\subsection{Results}

\paragraph{Segmentation with detected edges} 
For this task, we adopt {Adaptive foreground-background classification} (AdaFBC)~\cite{chen2012adaptive} as our foreground-background model. We use the foreground-background probability map produced by AdaFBC combined with feature based superpixel MRF (SSP-GC), active visual segmentation model~\cite{ajayTPAMI12SegFixation} (Aseg), and our method. Due to the page limit, we only present two set of visual results in Fig.~\ref{FIG:ExpReg}, additional results can be found in the supplementary material. Note that our method only requires one dot seed to achieve a satisfactory segmentation in those two examples, while the other methods require either tedious user interactions or produce visually noticeable artifacts. We also present the quantitative results of this experiment in Fig.~\ref{FIG:RegCmp} and the computation time in Table. \ref{TB:TcostIseg}. We can observe that our method is about 400 times faster than the original pixel-level method~\cite{ajayTPAMI12SegFixation}. Our method is also faster than SSP-GC. This is perhaps because the sparse edge map gives good contrast to the MRF weights in our model, which makes the inference much easier and faster.
\begin{figure}
  \centering
  \vspace{0.2cm}
  \includegraphics[width=0.9\columnwidth]{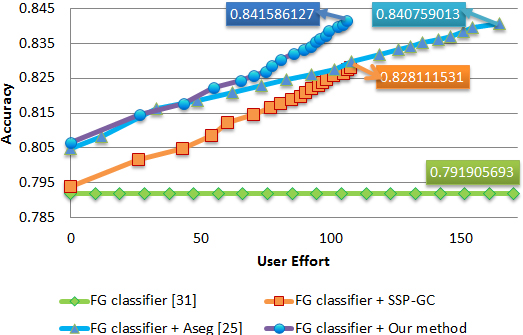}
  \caption{Ground truth comparison of segmentation score v.s. user efforts with initial bounding boxes.}\label{FIG:RegCmp}
\end{figure}

\begin{table}
\caption{Computation time for interactive segmentation (seconds per image).}\label{TB:TcostIseg}
\def\arraystretch{1.2}
{\small\begin{tabular*}{0.95\columnwidth}{@{\extracolsep{\fill} }l|ccccc}
\cline{1-6}
\textbf{method} & \textbf{mean} & \textbf{std} & \textbf{min} & \textbf{median} & \textbf{max} \\ \cline{1-6}
ASeg \cite{ajayTPAMI12SegFixation} &  0.087 & 0.0019 & 0.086 & 0.087 & 0.094 \\
SSP-GC & 0.0081 & 0.0014 & 0.0063 & 0.008 & 0.012  \\
Our method & 0.0026 & 0.00024 & 0.0024 & 0.0026 & 0.0036\\ \cline{1-6}
\end{tabular*}}
\end{table}

\begin{figure*}
\centering
\vspace{0.1cm}
\begin{tabular}{
@{\hspace{0mm}}c@{\hspace{0mm}}c@{\hspace{0mm}}c @{\hspace{0mm}}c
@{\hspace{0mm}}c@{\hspace{0mm}}c@{\hspace{0mm}}c @{\hspace{0mm}}c
@{\hspace{0mm}}c@{\hspace{0mm}}c
}
\begin{sideways}\parbox{19mm}{\centering\scriptsize FBC~\cite{Zhong2012UDC_SIGGRAPHAsia} }\end{sideways}&
          \includegraphics[height = 1.85cm]{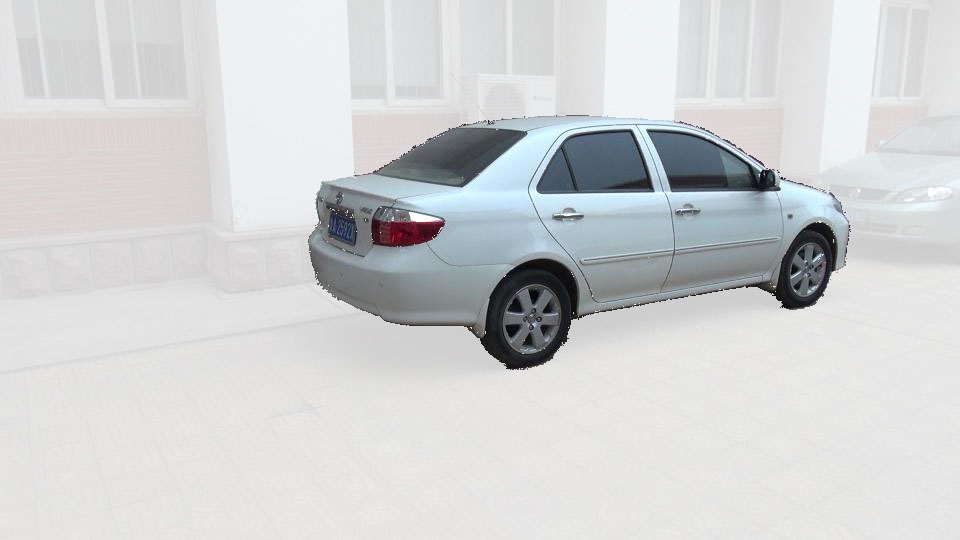}&
          \includegraphics[height = 1.85cm]{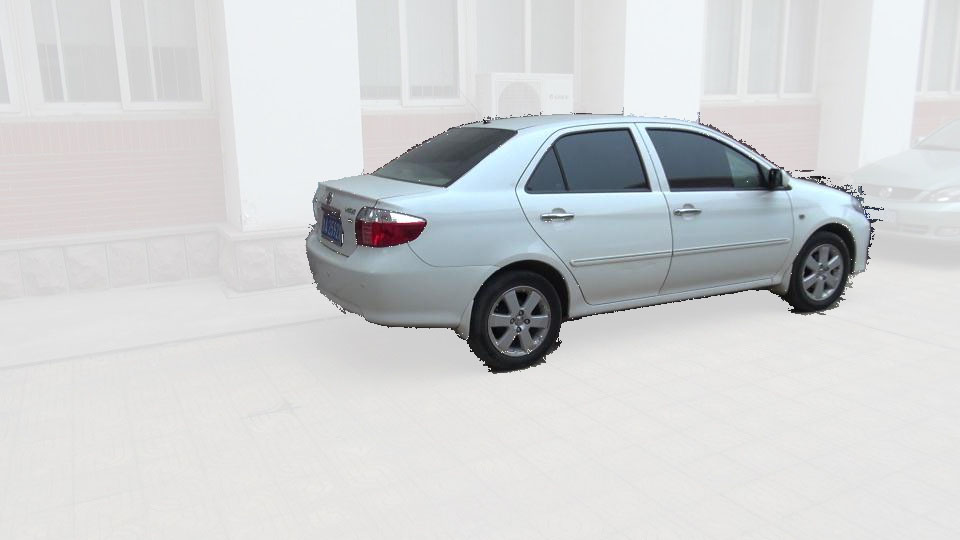}&
          \includegraphics[height = 1.85cm]{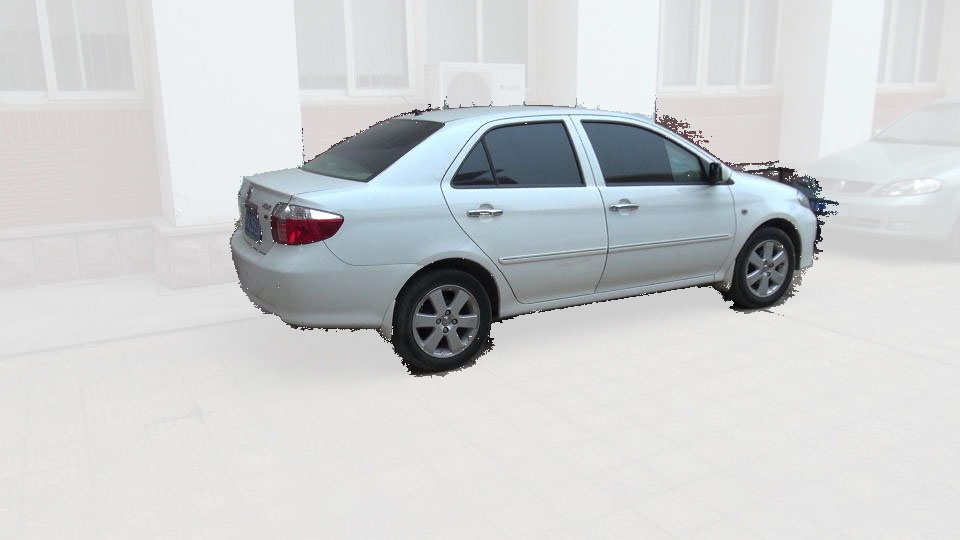}&
          \includegraphics[height = 1.85cm]{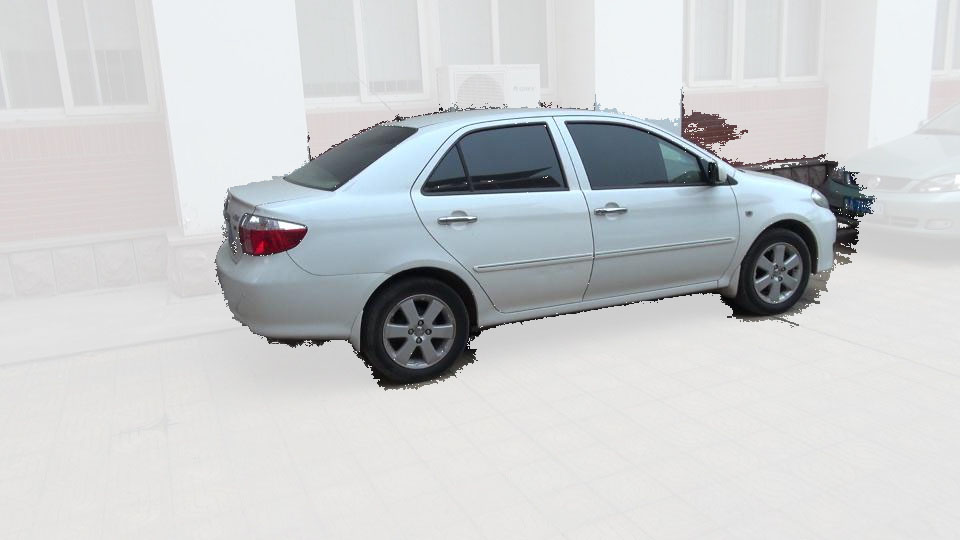}&
         \includegraphics[height = 1.85cm]{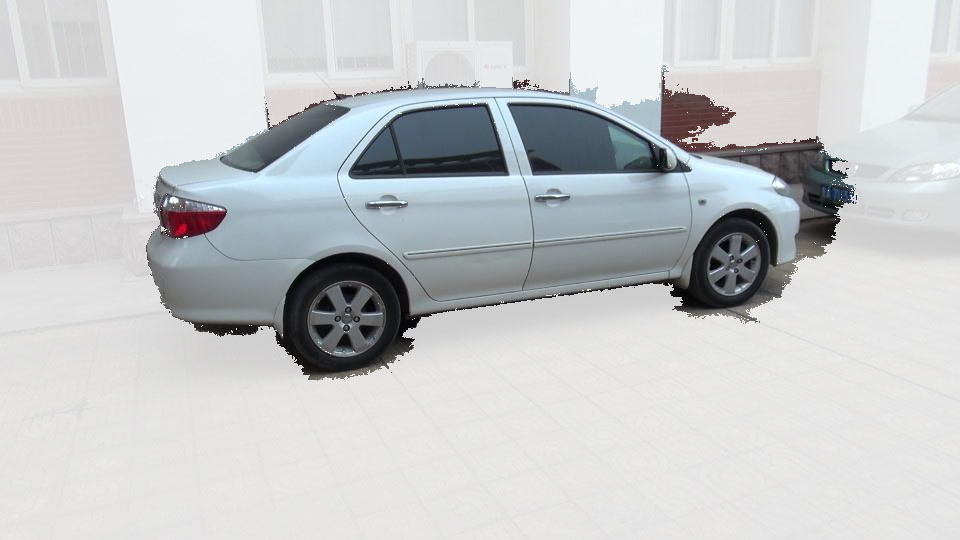}\\
\begin{sideways}\parbox{19mm}{\centering\scriptsize FBC + Matting~\cite{Zhong2012UDC_SIGGRAPHAsia} }\end{sideways}&
          \includegraphics[height = 1.85cm]{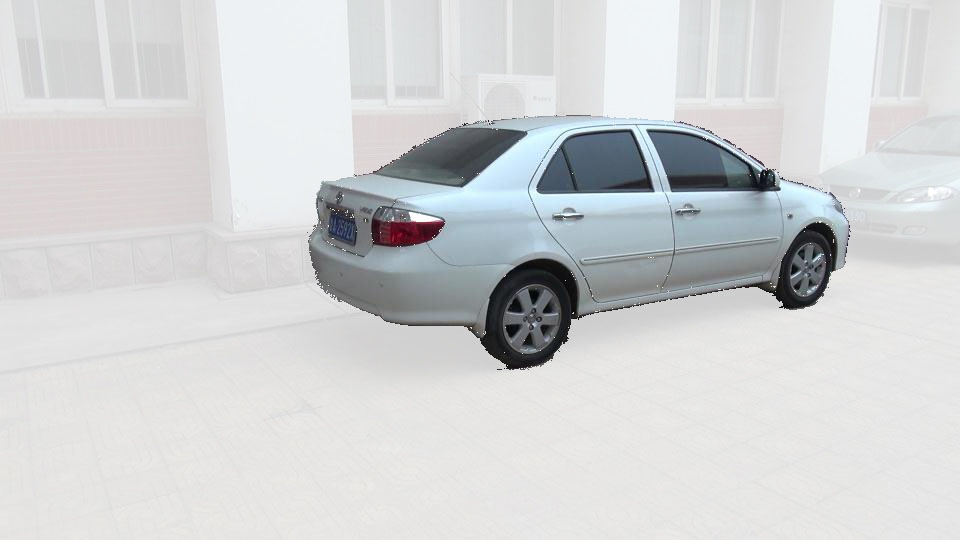}&
          \includegraphics[height = 1.85cm]{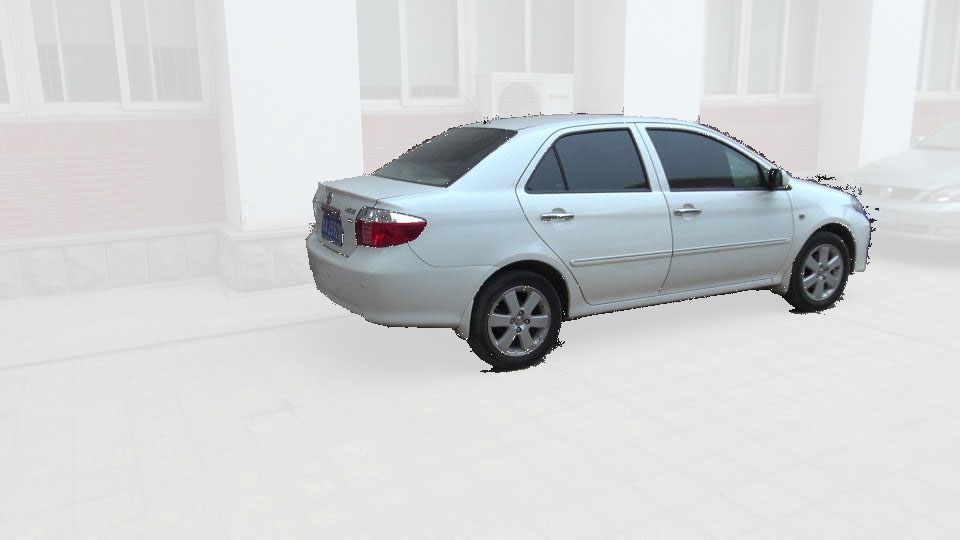}&
          \includegraphics[height = 1.85cm]{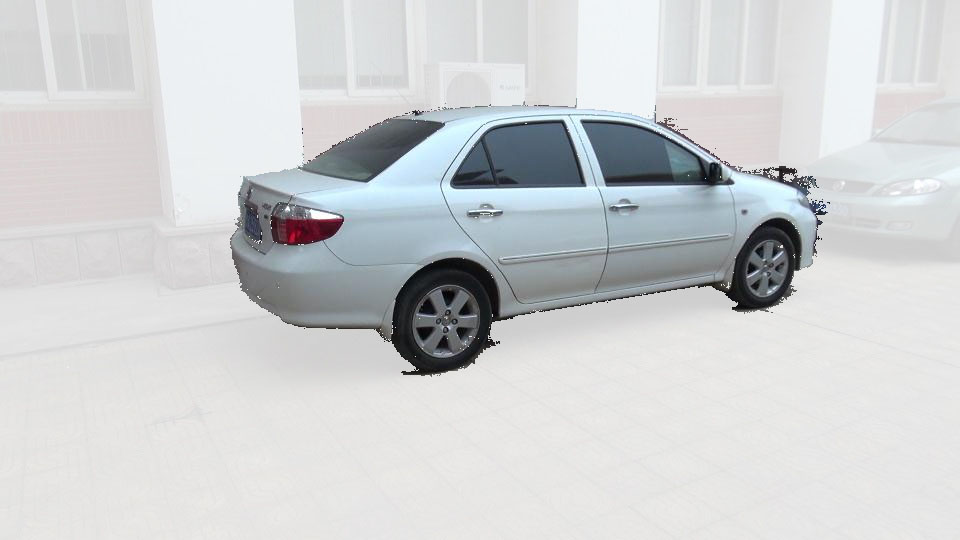}&
          \includegraphics[height = 1.85cm]{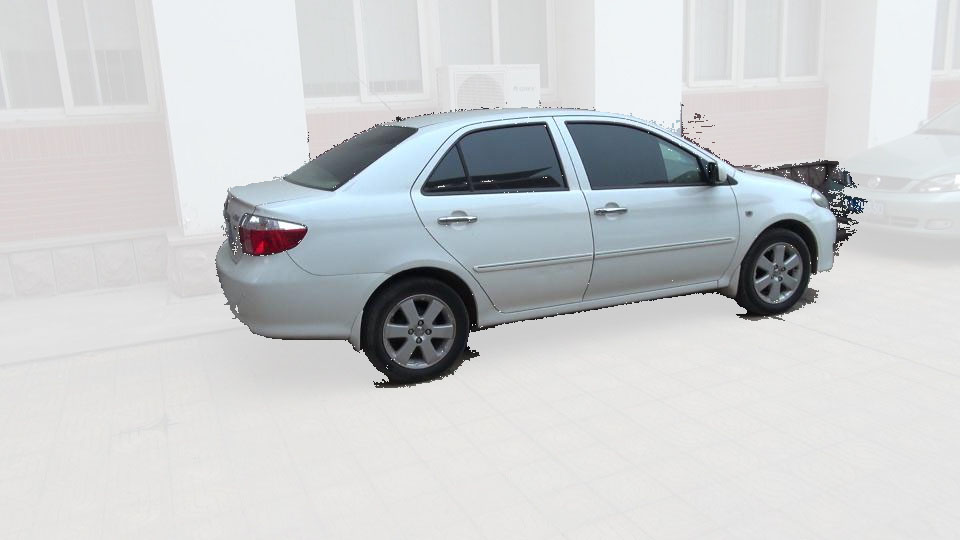}&
         \includegraphics[height = 1.85cm]{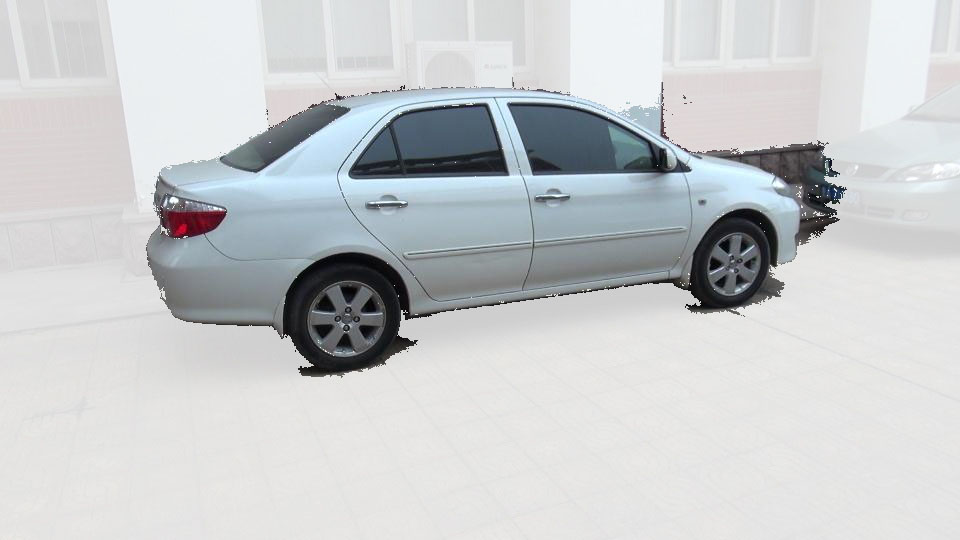}\\
\begin{sideways}\parbox{19mm}{\centering\scriptsize FBC + GC~\cite{Bai09VideoSnapCut_SIGGRAPH} }\end{sideways}&
          \includegraphics[height = 1.85cm]{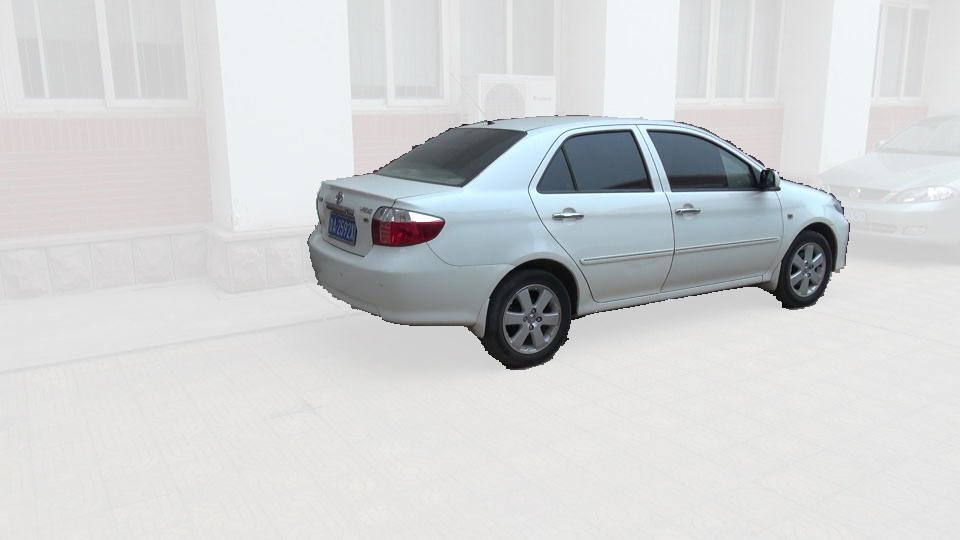}&
          \includegraphics[height = 1.85cm]{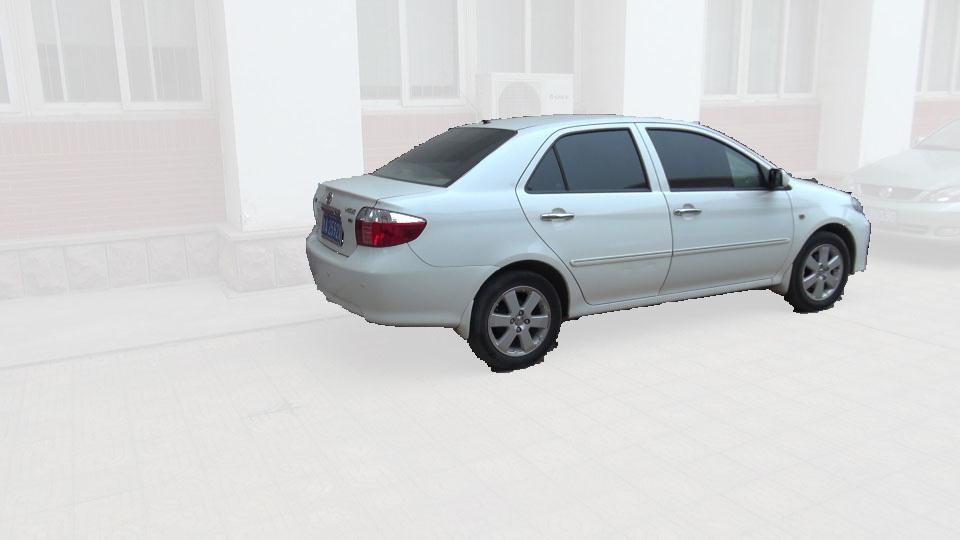}&
          \includegraphics[height = 1.85cm]{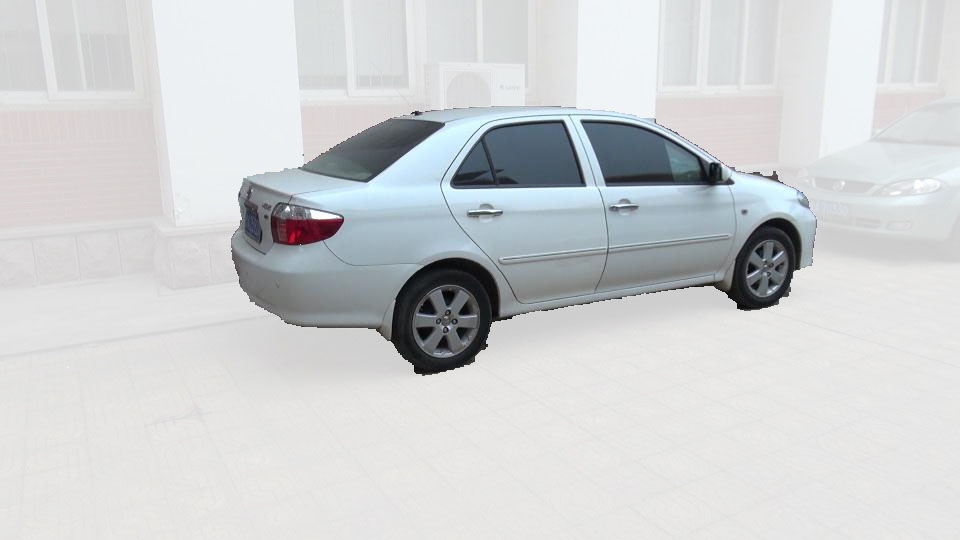}&
          \includegraphics[height = 1.85cm]{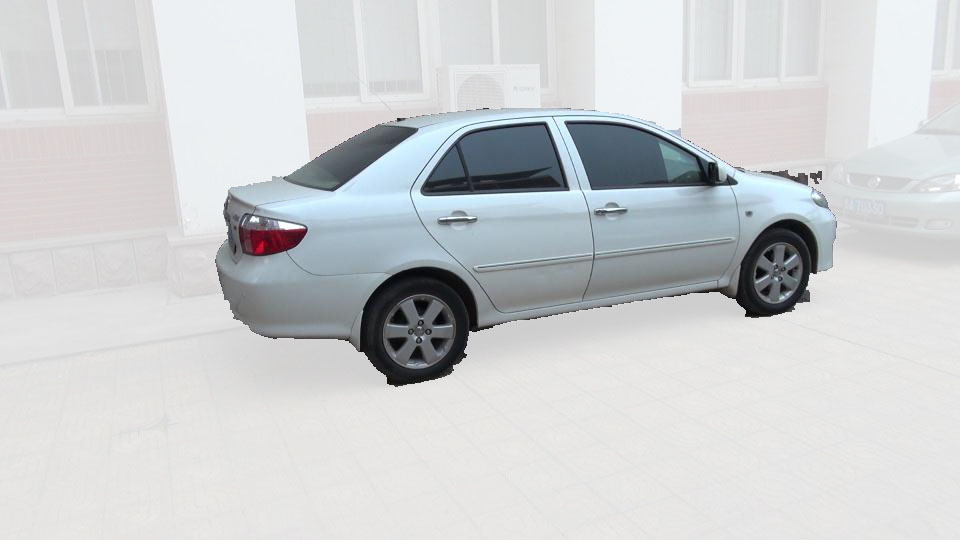}&
         \includegraphics[height = 1.85cm]{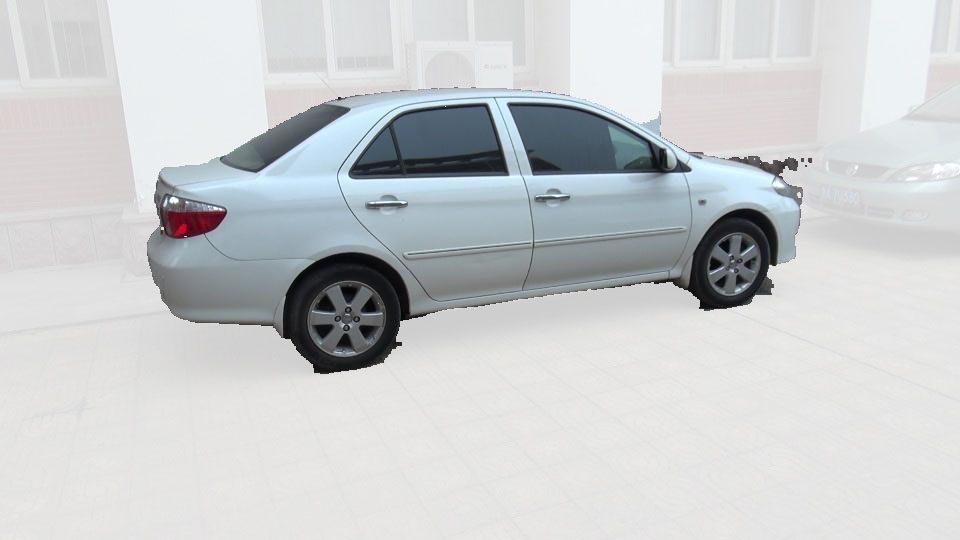}\\
\begin{sideways}\parbox{19mm}{\centering\scriptsize FBC + our method }\end{sideways}&
          \includegraphics[height = 1.85cm]{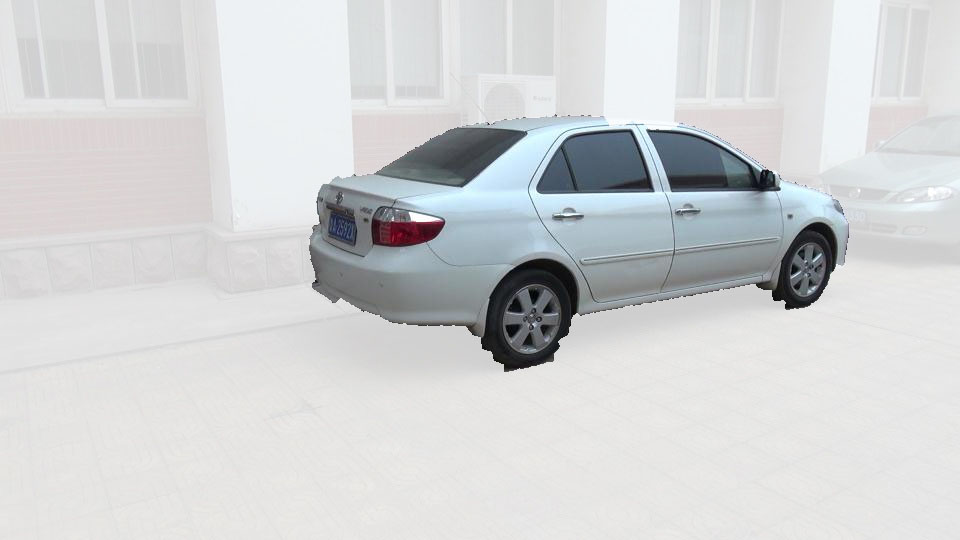}&
          \includegraphics[height = 1.85cm]{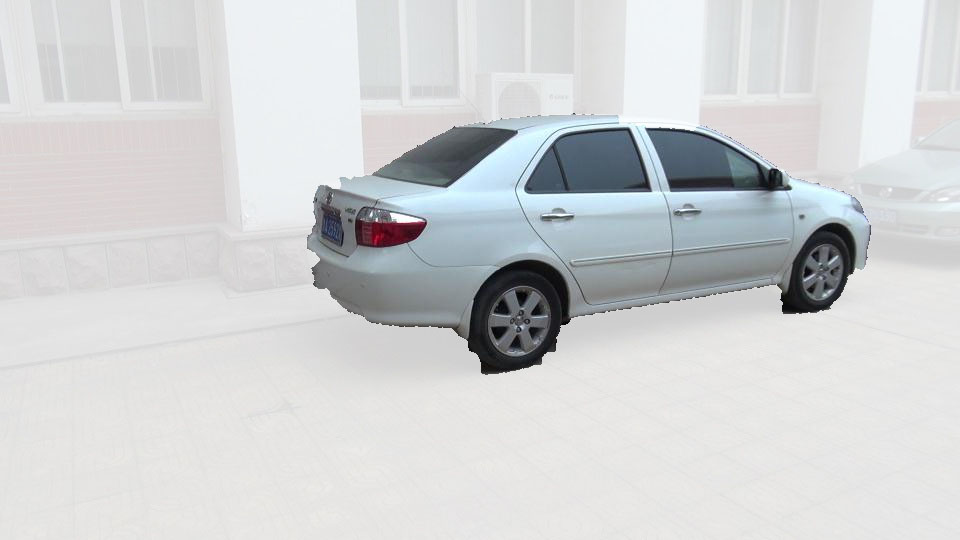}&
          \includegraphics[height = 1.85cm]{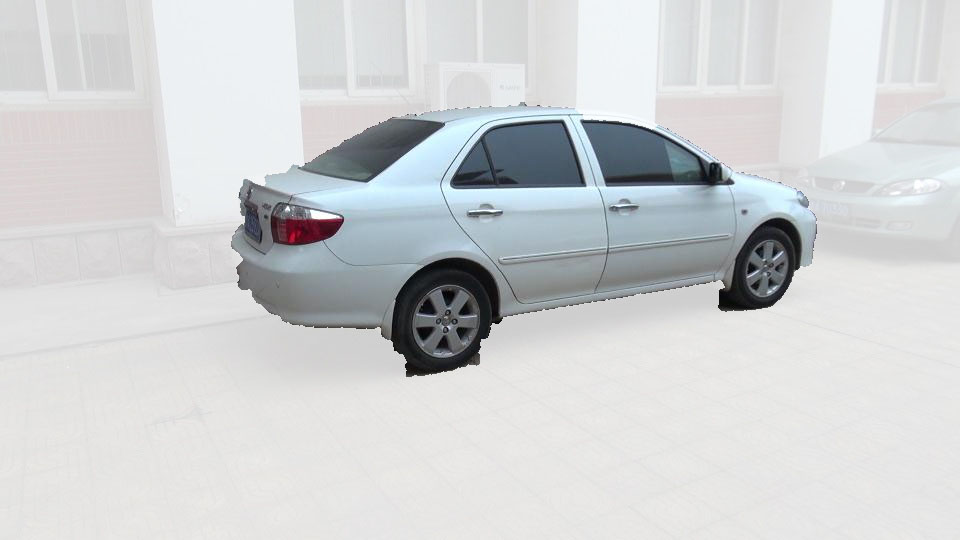}&
          \includegraphics[height = 1.85cm]{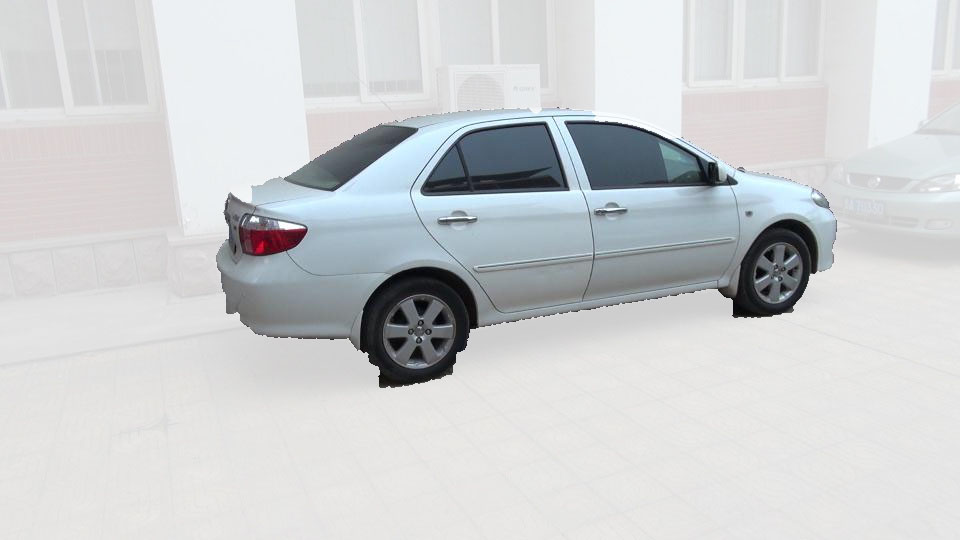}&
         \includegraphics[height = 1.85cm]{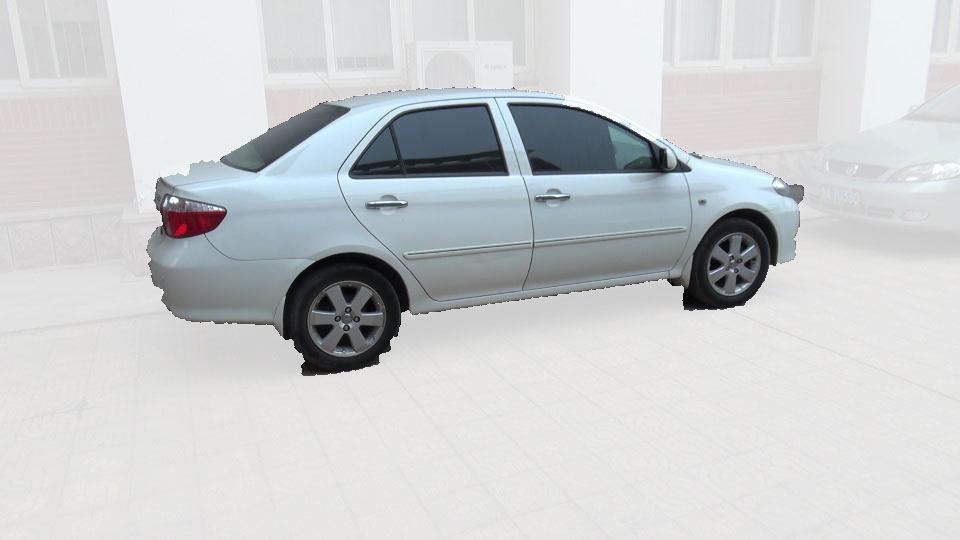}\\
         &Frame 1&Frame 3&Frame 5&Frame 7&Frame 9\\
\end{tabular}          
  \caption{Results of segmentation propagation for the ``Car'' sequence given the segmentation in the first frame. The background is whitened for visualization.}\label{FIG:ExampleSeqCar}
\end{figure*}

\paragraph{Segmentation propagation in video cutout}

In this experiment, we compare our method with the segmentation propagation adopted in \cite{Zhong2012UDC_SIGGRAPHAsia} and \cite{Bai09VideoSnapCut_SIGGRAPH}. The latter is known as Rotobrush in Adobe After Effect. We adopt the state-of-the art foreground-background classifier proposed in \cite{Zhong2012UDC_SIGGRAPHAsia} to form the unary term in the MRF. While Zhong et al. \cite{Zhong2012UDC_SIGGRAPHAsia} adopted matting for segmentation propagation, the Rotobrush uses graph cuts to solve a conventional MRF based segmentation model. In our implementation, we still adopt the model proposed in~\cite{ajayTPAMI12SegFixation} for this task. We present one set of visual results in Fig.~\ref{FIG:ExampleSeqCar}, more results can be found in the supplementary materials. The quantitative results are summarized in Fig.~\ref{FIG:ErrAcum}. From the results, we can observe that our method achieves the state-of-the-art segmentation propagation results. The advantage of our method lies in the computational efficiency, as tabulated in Tab. \ref{TAB:TcostVseg}.

\begin{figure}[!h]
\centering
\includegraphics[width = 0.9\columnwidth]{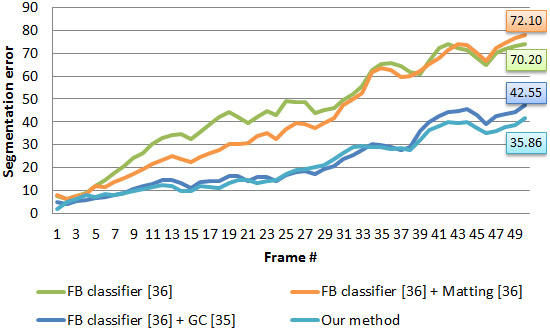}
\caption{Error accumulation in segmentation propagation.}\label{FIG:ErrAcum} \vspace{-0pt}
\end{figure}
\begin{table}[!h]
\vspace{0.1cm}
\caption{Computation time for segmentation propagation (seconds per frame).}\label{TAB:TcostVseg}
\def\arraystretch{1.2}
{\small\begin{tabular*}{0.95\columnwidth}{@{\extracolsep{\fill} }l|ccccc}
\cline{1-6}
\textbf{method} & \textbf{mean} & \textbf{std} & \textbf{min} & \textbf{median} & \textbf{max} \\ \cline{1-6}
Matting~\cite{Zhong2012UDC_SIGGRAPHAsia}
 & 0.96 &     0.3 &    0.42 &   0.92  &    1.8\\ 
GC~\cite{Bai09VideoSnapCut_SIGGRAPH}  & 0.38 &    0.12 &    0.22 &   0.35  &   0.73\\ 
Our method & 0.003 &  0.0002 &  0.0026 & 0.0029  & 0.0036\\ \cline{1-6}
\end{tabular*}}
\end{table}

\paragraph{Automatic segmentation proposal generation}
\begin{figure}
\centering
\begin{tabular}{
@{\hspace{0mm}}c@{\hspace{0mm}}c@{\hspace{0mm}}c @{\hspace{0mm}}c
@{\hspace{0mm}}c@{\hspace{0mm}}c@{\hspace{0mm}}c @{\hspace{0mm}}c
@{\hspace{0mm}}c@{\hspace{0mm}}c
}
\begin{sideways}\parbox{2.5cm}{\centering\scriptsize LGS \cite{Rantalankila2014GenProposal}}\end{sideways}&
          \includegraphics[height = 2.5cm]{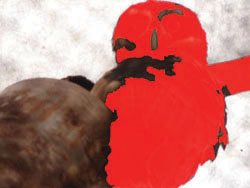}&
          \includegraphics[height = 2.5cm]{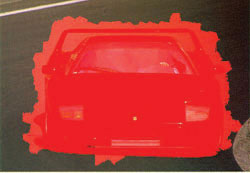}&
          \includegraphics[height = 2.5cm]{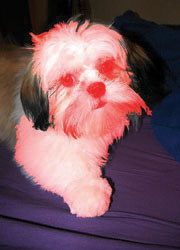}&
          \includegraphics[height = 2.5cm]{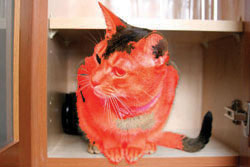}&
          \includegraphics[height = 2.5cm]{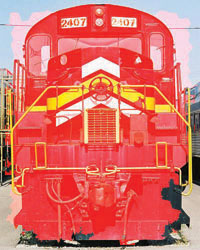}\\
\begin{sideways}\parbox{2.5cm}{\centering\scriptsize Our method }\end{sideways}&
          \includegraphics[height = 2.5cm]{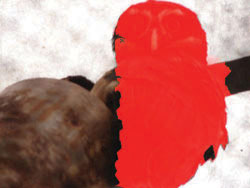}&
          \includegraphics[height = 2.5cm]{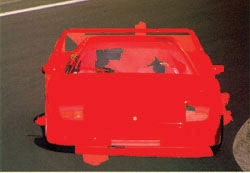}&
          \includegraphics[height = 2.5cm]{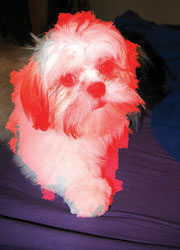}&
          \includegraphics[height = 2.5cm]{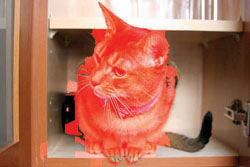}&
          \includegraphics[height = 2.5cm]{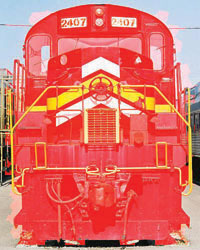}\\
\end{tabular}          
  \caption{Comparison of the maximum overlapping proposals}\label{FIG:ExampleProposalGen}
\end{figure}

\begin{figure}
  \centering
  \includegraphics[width=0.9\columnwidth]{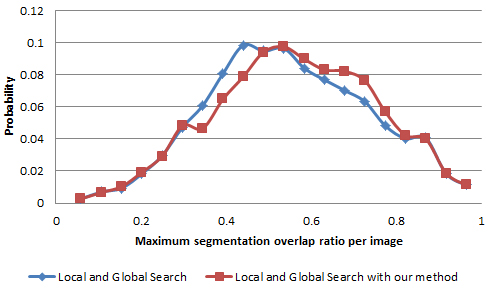}\\
  \caption{Performance of segmentation proposal generation. }\label{FIG:ProposalGen}
\end{figure}
Again, we only present some results of the segmentation proposal generation in Fig. \ref{FIG:ExampleProposalGen} due to the limit in paper length, additional results are in the supplementary materials. The visual results suggest that the results of our method better adheres to the object boundaries compared to the local and global search (LGS) method~\cite{Rantalankila2014GenProposal}. The quantitative results are shown in Fig. \ref{FIG:ProposalGen}. From the quantitative results we can observe that our method generates proposals of high accuracy at a higher probability. Table~\ref{TAB:TcostProposalGen} compare the computation time which is similar since both~\cite{Rantalankila2014GenProposal} and our method run on the superpixel space. 


\begin{table}[htb]
\vspace{0.1cm}
\caption{Computation time for segmentation proposal generation (seconds per image).}\label{TAB:TcostProposalGen}
\def\arraystretch{1.2}
\centering
{\small\begin{tabular*}{1\columnwidth}{@{\extracolsep{\fill}}l|ccccc}
\cline{1-6}
\textbf{method} & \textbf{mean} & \textbf{std} & \textbf{min} & \textbf{median} & \textbf{max} \\ \cline{1-6}
LGS~\cite{Rantalankila2014GenProposal} &
 10.46 & 2.19 & 1.38 & 10.59 & 20.33 \\ 
Our method & 9.37 & 1.93 & 1.26 & 9.81 &20.43\\ \cline{1-6}
\end{tabular*}}
\end{table}
\vspace{-3mm}

\section{Conclusion and future work}\label{SEC:Concl}
In this paper, we propose a technique to convert the generic binary MRF defined on pixels to binary MRF defined on superpixels, which we called superpixelization of MRF. The resultant model remains submodular if the original model is submodular. We applied the technique to several computer vision problems, and we either outperform the state-of-the-art at similar computational cost or we achieve the state-of-the-art at significantly smaller computational cost. Our technique is also potentially useful in solving non-submodular energy or multi-label problems and it is ready for extending to voxel labeling.


\section*{Appendix}
\renewcommand\thesubsection{\Alph{subsection}}
\setcounter{subsection}{0}
\renewcommand{\theequation}{A-\arabic{equation}}
\setcounter{equation}{0}  

\begin{proof}[Proof of lemma \ref{LM:uniq_SP}]
Let's consider $f_p$ defined in Eq. (\ref{EQ:SPlabel}). We will have 
\begin{equation}
f_p = f_p\chi^l_p,\hbox{ for } p\in{\Omega}_l.
\end{equation}
Substituting Eq. (\ref{EQ:SPlabel}) into the above, we will have for any $p\in{\Omega}_l$
\begin{equation}
f_p = \sum_{k=1}^{K} x_k \underbrace{\chi^k_p\chi^l_p}_{=0,~if~k\neq l}= x_l \chi^l_p.
\end{equation}
Note that $\chi^l_p=1$ for any $p\in{\Omega}_l$, $f_p = x_l$. This completes the proof.\qed
\end{proof}

\begin{proof}[Proof of corollary \ref{Col:invL}]
According to Lemma \ref{LM:uniq_SP}, we have for any $p\in{\Omega}_l$, $f_p =x_l$. Thus $\overline{f_p} = \overline{x_{l}} =\overline{x_{l}}\chi^{l}_p$, for any $p\in{\Omega}_l$. Thus for all $p\in\mathcal{P}$, we will have $\overline{f_p}=\sum_{k=1}^{K}\overline{x_k}\chi^k_p$. \qed
\end{proof}

\begin{proof}[Proof of Proposition \ref{PROP:main}]
We may start by expanding Eq. (\ref{EQ:PMRF_SP}). Accordingly, the unary term in Eq. (\ref{EQ:MRF}) can be rewritten using $x_k$:
\begin{equation}
\hspace{-10pt}E^1 = \sum_p w_p f_p = \sum_{k=1}^K\left(\sum_p w_p\chi^k_p\right) x_k=\sum_{k=1}^K \omega_kx_k,
\end{equation}
where $\omega_k = \sum_p w_p\chi^k_p$.

To superpixelize the pairwise potential of the MRF energy in Eq. (\ref{EQ:MRF}), we need to superpixelize the four pairwise terms: $\sum\limits_{(p,q)\in \mathcal{N}}  w^{00}_{pq}\overline{f_p}\overline{f_q}$, $\sum\limits_{(p,q)\in \mathcal{N}} w^{01}_{pq}\overline{f_p}{f_q}$,  $\sum\limits_{(p,q)\in \mathcal{N}}w_{pq}^{10}{f_p}\overline{f_q}$, $\sum\limits_{(p,q)\in \mathcal{N}}w_{pq}^{11}{f_p}{f_q}$.

Therefore, we have the following identities.
\begin{equation}
\begin{split}
&\sum_{(p,q)\in \mathcal{N}}  w^{00}_{pq}\overline{f_p}\overline{f_q}\\
=& \sum_{(p,q)\in \mathcal{N}} w^{00}_{pq}\sum_{k=1}^{K} \overline{x_k} \chi_p^k \sum_{l=1}^{K} \overline{x_l}\chi_q^l\\
=& \sum_{k=1,l=1}^{K,K} \left(\sum_{(p,q)\in \mathcal{N}} w^{00}_{pq} \chi_p^k \chi_q^l\right)\overline{x_k}\overline{x_l}\\
=& \sum_{k=1}^K\omega^{00}_k\overline{x_k} + \sum_{\{k=1,l=1|k\neq l\}}^{K,K} \omega_{kl}^{00}\overline{x_k}\overline{x_l}\\
=& \sum_{k=1}^K-\omega^{00}_k{x_k} + \hspace{-0.5cm}\sum_{\{k=1,l=1|k\neq l\}}^{K,K} \omega_{kl}^{00}\overline{x_k}\overline{x_l}+\sum_{k=1}^K\omega^{00}_k,\\
\end{split}
\end{equation}
where $\omega_{kl}^{00} = \sum\limits_{(p,q)\in \mathcal{N}} w^{00}_{pq} \chi_p^k \chi_q^l$, and $\omega_{k}^{00} = \sum\limits_{(p,q)\in \mathcal{N}} w^{00}_{pq} \chi_p^k \chi_q^k$. Note that $\chi_p^k \chi_q^l=1$ only when the neighboring $p,q$ are in two neighboring superpixels, i.e.  $p\in\Omega_k$ and $q\in\Omega_l$ for $k\neq l$.
\vspace{2pt}

Likewise,
\begin{equation}
\sum_{(p,q)\in \mathcal{N}} w^{01}_{pq}\overline{f_p}{f_q} = \sum_{\{k=1,l=1|k\neq l\}}^{K,K} \omega_{kl}^{01}\overline{x_k}{x_l}~~~~~~~~~~
\end{equation}
\begin{equation}
\sum_{(p,q)\in \mathcal{N}} w^{10}_{pq}{f_p}\overline{f_q} = \sum_{\{k=1,l=1|k\neq l\}}^{K,K} \omega_{kl}^{10}{x_k}\overline{x_l}~~~~~~~~~~
\end{equation}
\begin{equation}
\sum_{(p,q)\in \mathcal{N}} w^{11}_{pq}{f_p}{f_q} =\sum_{k=1}^{K} \omega_{k}^{11}{x_k}+\sum_{\{k=1,l=1|k\neq l\}}^{K,K} \omega_{kl}^{11}{x_k}{x_l},
\end{equation}
where $\omega_{kl}^{mn} = \sum\limits_{(p,q)\in \mathcal{N}} w^{mn}_{pq} \chi_p^k \chi_q^l$, $\omega_{k}^{mn} = \sum\limits_{(p,q)\in \mathcal{N}} w^{mn}_{pq} \chi_p^k \chi_q^k$, where $(m,n)\in \{0,1\}$. Note that there are no linear terms for $m=1,~n=0$ and $m=0,~n=1$, since $x_k\overline{x_k}=0$.
\vspace{2pt}

To sum up, the superpixelized MRF energy can be rewritten as
\begin{equation}
\begin{split}
&E^1+E^2\\
=& \sum_{k=1}^K \hat{\omega}_kx_k+\sum_{k\neq l}^{K,K}\left( \omega_{kl}^{00}\overline{x_k}\overline{x_l}+\omega_{kl}^{01}{x_k}\overline{x_l}\right.\\ 
&\hspace{2.5cm}\left.+\omega_{kl}^{10}{x_k}\overline{x_l}+\omega_{kl}^{11}{x_k}{x_l}\right)\\
&+C,
\end{split}
\end{equation}
where $C$ is a constant independent of $x_k$, $\hat{\omega}_k = \omega_k - \omega^{00}_{k} + \omega^{11}_{k} $, and the remaining variables are defined as before. The resultant form turns out to be analogous to the original pixel-level MRF. Note that $\omega^{mn}_{k}$ is treated as $0$ for $m=1,~n=0$ and $m=0,~n=1$, since $x_k\overline{x_k}=0$.
\qed
\end{proof}

\begin{proof}[Proof of corollary \ref{Col:V}]
First, we can take summation of $V_{pq}$ over the neighborhood defined by $\mathcal{N}_{kl}=\{\{p,q\}\in \mathcal{N}|p\in\Omega_k, q\in\Omega_l,k\neq l\}$ to arrive at the following:
\begin{equation}
\begin{split}
\hspace{-0.3cm}&\sum_{\{p,q\}\in\mathcal{N}_{kl}}V_{pq}\\
\hspace{-0.2cm}&=\hspace{-0.3cm}\sum_{\{p,q\}\in\mathcal{N}_{kl}}\hspace{-0.2cm}w^{00}_{pq}\overline{f_p}\overline{f_q}+w^{01}_{pq}\overline{f_p}{f_q}+w^{10}_{pq}{f_p}\overline{f_q}+w^{11}_{pq}{f_p}{f_q}\\
\end{split}
\end{equation}
According to lemma \ref{LM:uniq_SP} and corrollary \ref{Col:invL}, the above can be written as:
\begin{equation}\label{EQ:intermV}
\begin{split}
\hspace{-0.3cm}&\sum_{\{p,q\}\in\mathcal{N}_{kl}}V_{pq}\\
\hspace{-0.2cm}&=\hspace{-0.3cm}\sum_{\{p,q\}\in\mathcal{N}_{kl}}\hspace{-0.2cm}w^{00}_{pq}\overline{x_k}\overline{x_l}+w^{01}_{pq}\overline{x_k}{x_l}+w^{10}_{pq}{x_k}\overline{x_l}+w^{11}_{pq}{x_k}{x_l}\\
\end{split}
\end{equation}

From proposition \ref{PROP:main}, we know that
\begin{equation}
\omega^{mn}_{kl} = \sum_{\{p,q\}\in\mathcal{N}_{kl}}w_{pq}^{mn}.
\end{equation}
By substituting the above into Eq. (\ref{EQ:intermV}), we obtain the LHS of Eq. (\ref{EQ:PPafSP}) which complete the prove.\qed
\end{proof}

\begin{proof}[Proof of proposition \ref{Prop:Regularity}]
Let us multiply each term of Eq. (\ref{EQ:regineq}) with $\chi_p^k\chi_q^l$, which is non-negative. We will have for any $(p,q)\in\mathcal{N}$,
\begin{equation}
w^{00}_{pq}\chi_p^k\chi_q^l+w^{11}_{pq}\chi_p^k\chi_q^l\leq w^{01}_{pq}\chi_p^k\chi_q^l+w^{10}_{pq}\chi_p^k\chi_q^l.
\end{equation}
If we further sum each term over all the $(p,q)\in\mathcal{N}$ together, we will have
\begin{equation}
\begin{split}
&\sum\limits_{(p,q)\in \mathcal{N}}(w^{00}_{pq}\chi_p^k\chi_q^l+w^{11}_{pq}\chi_p^k\chi_q^l)\\
&\leq \sum\limits_{(p,q)\in \mathcal{N}}(w^{01}_{pq}\chi_p^k\chi_q^l+w^{10}_{pq}\chi_p^k\chi_q^l).
\end{split}
\end{equation}
By definition of $\omega^{00}_{kl}$,~$\omega^{11}_{kl}$, $\omega^{01}_{kl}$, and $\omega^{10}_{kl}$, the above completes the proof.\qed
\end{proof}

\balance
{\small
\bibliographystyle{IEEEtran}
\bibliography{MRFseg,LowLevelVision,LevelSetActiveContours,OtherSeg,VideoSeg}
}

\newpage

\end{document}